\documentclass[article, 11pt, reqno]{amsart}  
\usepackage{mathtools, amsmath, amsthm, amsfonts, amssymb} 
\usepackage{amssymb}
\usepackage{graphicx} 
\usepackage[top=3cm, bottom=3cm, left=4cm, right=4cm]{geometry} 
\usepackage{csquotes} 
\usepackage{epigraph} 
\usepackage{natbib} 
\setcitestyle{aysep={}}
\usepackage[toc,page]{appendix} 
\usepackage[hidelinks]{hyperref}
\hypersetup{
    colorlinks=false, 
    linktoc=all,  
    linkcolor=black,  
}
\usepackage{tikz}
\usetikzlibrary{shapes.geometric}
\newtheorem{theorem}{Theorem}

\newtheorem{definition}[theorem]{Definition}

\newtheorem{assumption}[theorem]{Assumption}

\newtheorem{corollary}[theorem]{Corollary}

\newtheorem{proposition}[theorem]{Proposition}

\newtheoremstyle{named}{}{}{\itshape}{}{\bfseries}{.}{.5em}{\thmnote{#3}#1}
\theoremstyle{named}

\DeclareMathOperator{\R}{\mathbb{R}}
\DeclareMathOperator{\E}{\mathbb{E}}
\DeclareMathOperator{\Var}{Var}
\DeclareMathOperator{\Cov}{Cov}
\DeclareMathOperator{\Corr}{Corr}
\usepackage{booktabs} 
\usepackage{makecell} 
\usepackage{subcaption}
\usepackage{comment}

\begin{document}

\title{Learning the Value of Value Learning}
\author{Alex John London \& Aydin Mohseni}
\maketitle

\begin{abstract}
Decision theory provides well-understood guarantees for the value of acquiring information about states of the world, but assumes that agents have fixed options and settled values. Real agents---human, institutional, and artificial---frequently face uncertainty about the nature and scope of their own values. We extend the Jeffrey--Bolker framework to model value refinement and prove a value-of-information theorem for axiological refinement: under broad conditions, pausing to refine one's values has strictly positive expected utility. In multi-agent settings, we establish that unilateral value refinement transforms zero-sum games into positive-sum interactions and yields Pareto improvements in Nash bargaining. By unifying epistemic and axiological refinement under a single formalism, we reveal a two-stage structure of rational choice: refine, then maximize.
\end{abstract}

\vspace*{10pt}
\noindent{\footnotesize\textbf{Keywords:} decision theory, rational choice, formal epistemology, preference change, unawareness, awareness growth, model misspecification}

\vspace*{10pt}

\section{Introduction}

Consider two types of uncertainty that agents face. One is \textit{epistemic}: uncertainty about states of the world. The other is \textit{axiological}: uncertainty about the nature and scope of one's own values. Value-of-information theorems establish that, under broad conditions,\footnote{Specifically, when there is act-state independence, information is costless, and agents have countably additive, proper priors \citep{KadaneSchervishSeidenfeld2008,Skyrms1984Pragmatics,GibbardHarper1978}.} rational agents benefit from reducing epistemic uncertainty---costless evidence cannot lower expected utility \citep{Good1967}. No analogous result exists for axiological uncertainty. We provide one: pausing to refine one's values, under quite general conditions, has strictly positive expected utility.

Value refinement is not reducible to epistemic refinement. In strategic settings where value-of-information guarantees break down \citep{Hirshleifer1971, GibbardHarper1978, Skyrms1984Pragmatics}, value refinement still yields expected gains. In two-player zero-sum games, value refinement in expectation transforms interactions into positive-sum ones. In Nash bargaining, value refinement in expectation produces Pareto improvements for all parties. These results establish that there is strategic, not merely epistemic, value in clarifying one's values.

The results carry consequences for rational choice theory. Standard arguments for expected utility maximization---representation theorems \citep{VonNeumannMorgenstern1944, Savage1954, Jeffrey1965}, Dutch book arguments \citep{Ramsey1931, DeFinetti1937, Kemeny1955}, accuracy arguments \citep{Joyce1998, Pettigrew2016, Schoenfield2017}, and consequentialist arguments \citep{Hammond1988, McClennen1990}---presume that the agent has a stable set of values representable by a unique (up to positive affine transformation) utility function. Yet agents frequently face uncertainty not just about external states, but about their own options and valuations. This raises a question that existing frameworks leave unanswered: when should an agent treat their current utility function as authoritative and act, and when should they instead pause to refine their values?

We show that foregoing value refinement can be strictly dominated by refining first. Rational choice therefore has a two-stage character: \emph{refine, then maximize}. By providing a formal model of value refinement within decision theory, we make explicit a connection between ethical reflection and rational choice that is often left implicit.

In \S\ref{sec:Background}, we situate the work in existing literature. \S\ref{sec:Setup} introduces the Jeffrey--Bolker decision framework and our model of value refinement. In \S\ref{sec:valueRefinement1A}, we prove a value-of-value-refinement theorem for a single agent. In \S\ref{sec:valueCommensuration1A}, we show how value refinement can dissolve apparent dilemmas involving incommensurable values. We extend the analysis to strategic settings in \S\ref{sec:valueRefinement2A0-Sum}, demonstrating that mutual value refinement turns zero-sum conflict into positive-sum cooperation, and in \S\ref{sec:valueRefinementNashBargaining}, where we derive Pareto improvement guarantees in Nash bargaining. In \S\ref{sec:Discussion}, we discuss normative and methodological implications and outline future directions.


\section{Background}\label{sec:Background}

\subsection*{Value of Information}

Good's \citeyearpar{Good1967} theorem---that costless evidence cannot reduce a rational agent's expected utility---is a cornerstone of decision theory.\footnote{The result is anticipated by Ramsey's proof of a special case four decades earlier \citeyearpar{ramsey1990weight}.} The theorem requires act-state independence, dynamic consistency, and countably additive proper priors. It fails when any of these conditions is relaxed: Hirshleifer \citeyearpar{Hirshleifer1971} showed that others' responses to information acquisition can reduce welfare; Skyrms \citeyearpar{Skyrms1990} identified the role of dynamic consistency; Kadane, Schervish, and Seidenfeld \citeyearpar{KadaneSchervishSeidenfeld2008} demonstrated failures with improper or non-countably additive priors and imprecise credal states; and Seidenfeld \citeyearpar{Seidenfeld2009} isolated act-state dependence as an independent source of failure. Under a broad and well-understood range of conditions, however, rational agents should expect to do at least as well by acquiring information before acting.

Value-of-information guarantees break down in game-theoretic contexts where act-state independence fails \citep{Skyrms1984Pragmatics}. When actions alter the relative probabilities of states, additional information can decrease expected utility. One might conclude that value refinement offers no advantage in strategic settings. We demonstrate otherwise: value refinement yields substantial expected benefits in game-theoretic contexts of extreme conflict.

\subsection*{Axiological Uncertainty}

In contrast to the extensive literature on epistemic uncertainty, axiological uncertainty---uncertainty about the nature, content, or scope of one's own values---remains underdeveloped as a formal topic despite its prevalence in real-world decision-making. \citet[pp. 172--176]{DeweyTufts1936Ethics} distinguish weakness of will, where agents struggle to act on their judgments, from cases where the agent is uncertain about how to reconcile competing values. In the latter case, they argue, what is required is reflection and clarification rather than strength of will. They provide no formal model of such reflection.

Levi \citeyearpar{levi1990hard} provides a formal model of choice under value uncertainty. He identifies two sources of value uncertainty: pragmatic constraints (insufficient time to reconcile values) and deep incommensurability (no unique tradeoff function exists). We add a third, normative source: some ethical theories hold that commensuration, though possible, is morally impermissible---Kantian theories, for instance, hold that treating rational agents as having a price violates their dignity \citep{LondonEtAl2017}. Levi's procedure requires agents to identify relevant values, entertain all convex combinations, and treat as admissible any act on or above the resulting surface, applying second-order values iteratively when multiple acts remain admissible. The resulting value structure is weaker than a utility function: agents acting on utility functions can be represented as employing a particular value structure, but Levi-rational choices need not correspond to any utility function.

Our approach differs. Levi models value inquiry as identifying relative weights among fixed values. We show that inquiry can alter values themselves, resolving intra-personal and interpersonal conflicts. Our results apply within Levi's framework (demonstrated in \S\ref{sec:valueCommensuration1A}) and extend to other approaches accommodating incommensurability, parity, or ambiguity.

\subsection*{Awareness Growth and Unawareness}

Two literatures are closest in spirit to ours: work on awareness growth in rational choice, and work on unawareness in game theory.

In the first, \cite{Bradley2017HumanFace} models awareness growth through expansion and refinement operations that partition existing propositions or introduce new sample space elements. These operations obey reverse Bayesianism, preserving the agent's original likelihood ratios. Bradley's refinement process parallels ours; we extend it to decision-theoretic contexts involving utility uncertainty, not just probability uncertainty. This extension requires relaxing strict reverse Bayesianism while observing more general Bayesian reflection principles that we detail below. Steele \& Stef\'{a}nsson \citeyearpar{steele2021beyond} propose two principles for awareness growth: awareness reflection (the agent's current credence should equal their expected credence conditional on anticipated awareness changes) and preference awareness reflection (rational agents should not expect their preference rankings to change). Our reflection principles satisfy both constraints while providing a formal account of how awareness growth is realized in individual and strategic settings. \cite{Pettigrew2024-BeliefChange} develops a related framework for belief change under awareness growth.

In game theory, Halpern and R\^{e}go \citeyearpar{HalpernRego2009ReasoningUnawareness,RegoHalpern2012GeneralizedSolutionConcepts,HalpernRego2013ReasoningUnawarenessRevisited} model agents who are unaware of all potentially available actions, analyze the implications for solution concepts, and argue that representing such agents requires augmenting standard game-theoretic machinery. Whereas they investigate what happens when agents cannot represent their options, we investigate when it is worthwhile for agents to refine their representation of options and evaluations, and what rational principles constrain the process.

\subsection*{Related Formal Approaches}

\cite{CyertDeGroot1975,CyertDeGroot1979} treat uncertainty about one's values as parametric inference---agents learn about an unknown value parameter via Bayesian updating. Our focus differs: we examine how refinement of one's decision representation, uncertainty about value, and rational principles jointly determine when to engage in value refinement.

Our results are distinct from problems of diachronic coherence, including how ``transformative" experiences can alter agents' preferences \citep{ullmann2006big, paul2014transformative, pettigrew2015transformative, pettigrew2019choosing}. The changes we model arise not from preference transformation but from more precise specification of commitments that an agent already endorses within particular decision contexts.

\subsection*{The Jeffrey--Bolker Framework}

The Jeffrey-Bolker decision framework \citep{Jeffrey1965, Bolker1966} is well-suited to this project. Unlike Savage's \citeyearpar{Savage1954} separation of states, acts, and outcomes, Jeffrey-Bolker employs propositions as the fundamental evaluation unit. Its atomless Boolean algebra structure allows preferences to be defined over propositions directly, with probability and desirability assigned to the same entities. This facilitates value refinement: agents subdivide coarse propositions into finer ones as understanding develops, and propositions in the algebra can be refined in well-defined ways. We suspect that analogous results to those presented here can be proven in the frameworks of Savage, von Neumann-Morgenstern, and others, though we do not pursue this here.


\section{Setup \& Formal Preliminaries}\label{sec:Setup}
\subsection{The Arborist and the Baker}\label{sec:Vignette}

Consider an arborist and a baker who jointly discover a rare orange in the wild. Both claim the fruit and find themselves in conflict over the prize.

Initially, each sees two possibilities: cooperate and divide the orange, or fight over the whole thing at the risk of getting none. Given their current understanding, the situation appears zero-sum---any gain for one party constitutes a loss for the other.

Before proceeding, both parties step back and refine their representation of the situation---through reflection or conversation. They discover that ``divide the orange'' contains distinct possibilities they had not initially distinguished. The arborist primarily values the fruit for its seeds---a means of propagating a rare specimen. The baker values the rind for its zest in specialty pastries.

This value refinement transforms how they conceptualize their options. Where before they saw only ``zero-sum split'' versus ``conflict,'' they now distinguish qualitatively different ways of dividing the resource: one where the arborist gets the seeds and the baker gets the rind, another where these assignments are reversed. One particular division---arborist takes seeds, baker takes rind---is strictly better for both than their initial conception of generic cooperation. What appeared to be a zero-sum conflict becomes a positive-sum arrangement: each party obtains a better outcome without leaving the other worse off.\footnote{This type of interaction is modeled in \S \ref{sec:valueRefinementNashBargaining} on value refinement in Nash bargaining games.}

When agents refine their options and valuations, they can reveal unrecognized dimensions of value and transform perceived zero-sum interactions into positive-sum arrangements, improving both individual and collective welfare. To make this case with precision, we begin with value refinement for a single agent and progress to strategic interactions.



\subsection{The Jeffrey--Bolker Decision Framework}
\label{sec:JB-framework}

We employ the Jeffrey-Bolker decision framework \citep{Jeffrey1965, Bolker1966}. Bolker's representation theorem establishes that if an agent's preferences over propositions satisfy standard rationality conditions---ordering, averaging, impartiality, and continuity---then there exist a finitely additive probability measure $P: \mathcal{A} \to [0,1]$ and utility function $U: \mathcal{A}\setminus \{\bot \} \to \mathbb{R}$ representing those preferences.\footnote{Unique up to positive affine transformation of $U$ with compensating positive scalar rescaling of $P$.}

The Jeffrey-Bolker framework has several features that make it well-suited to modeling value refinement. In Jeffrey-Bolker, acts, states, and consequences are all propositions in a single Boolean algebra. Assigning probabilities to propositions about one's own acts is essential for representing uncertainty about the structure of fine-grained actions that may result from value refinement. Collapsing the state/act/outcome distinction streamlines the analysis, because every potential more detailed description of an option is already present in the algebra.

An agent's representation of a decision situation is a tuple $\langle \mathcal{A}, \mathbb{A}, U, P \rangle$: an algebra $\mathcal{A}$ of relevant possibilities the agent entertains; a partition $\mathbb{A} \subseteq \mathcal{A}$ encoding the available acts; and functions $U$ and $P$ encoding utilities and credences over $\mathcal{A}$. In words, a decision representation specifies which possibilities the agent entertains, which propositions count as acts, and how the agent evaluates and assigns credences to each. 

Let $\mathcal{A}' = \mathcal{A} \setminus \{\bot\}$ denote the non-empty elements. $\mathcal{A}'$ is \textit{atomless} if for all $x \in \mathcal{A}'$, there exists $y \in \mathcal{A}'$ with $y \neq x$ and $y \subseteq x$. $\mathcal{A}$ is \textit{complete} if every subset has a supremum and infimum in $\mathcal{A}$. Atomlessness guarantees arbitrary precision: any $x \in \mathcal{A}'$ can be split into $x \land y$ and $x \land \neg y$, or more generally decomposed into finite partitions $x \equiv \bigvee_{i=1}^{n} y_i$. Since actions are propositions in such an algebra, any refinement of an agent's act partition can be the sensible objects of belief and value---that is, the domain of credence and utility functions.


\subsection{Modeling Value Refinement}

We extend the Jeffrey--Bolker framework by introducing a \emph{refinement operation} on an agent's representation of a decision problem. We begin with an idealized benchmark, then develop the machinery for bounded agents who can refine their understanding of the actions available to them.

\subsubsection{The Unbounded Benchmark}

Consider an unbounded agent who has evaluated all possible refinements. Their complete representation is $D = \langle \mathcal{A}, \mathbb{A}, P, U \rangle$ where $\mathcal{A}$ is a complete, atomless Boolean algebra containing all conceivable propositions; $\mathbb{A}$ is an act partition including all maximally fine-grained acts;\footnote{We consider the finite case, though one could extend to infinite act partitions (e.g., proper filters or ultrafilters).} and $P: \mathcal{A} \to [0,1]$ and $U: \mathcal{A}' \to \mathbb{R}$ are fully specified.\footnote{While an agent's algebra may be atomless, their act partition may be finite---reflecting a maximum level of precision in the control they can exercise, even if they can entertain finer distinctions regarding states.}

\subsubsection{The Bounded Agent}

The agents we study have not undergone arbitrary refinement. Their initial representation is $D_0 = \langle \mathcal{A}_0, \mathbb{A}_0, P_0, U_0 \rangle$ where $\mathcal{A}_0 \subset \mathcal{A}$ is a coarse-grained subalgebra; $\mathbb{A}_0$ is the maximally fine-grained coarsening of $\mathbb{A}$ within $\mathcal{A}_0$; and $P_0$ and $U_0$ are restricted to $\mathcal{A}_0$ and $\mathcal{A}_0'$, respectively.

How does such an agent refine their understanding of acts $\mathbb{A}_0$ and reason about propositions in $\mathcal{A} \setminus \mathcal{A}_0$ that they have not yet considered?

\subsubsection{Refinement Operations}

We model bounded agents as exploring their decision problem by considering available acts in greater detail. An agent takes a given act $A$ and considers distinct ways it might be realized with respect to previously unconsidered propositions $B_1$ and $B_2$. For example, our arborist might refine ``split the resource'' into ``split the resource and give the rind to the baker'' and ``split the resource and give the fruit to the baker.'' 

\begin{definition}[Binary Refinement]
Given an act partition $\mathbb{A}_0$ and $A \in \mathbb{A}_0$, a \emph{binary refinement} of $A$ produces the partition 
$$\mathcal{R}_{A} = \{A \land B_1, A \land B_2\}$$
where $B_1, B_2 \in \mathcal{A} \setminus \mathcal{A}_0$ partition $A$.\footnote{That is, $B_1 \lor B_2 = A$, $B_1 \land B_2 = \bot$, and neither $B_i$ is $\bot$.} This yields the refined algebra $\mathcal{A}_1 = \sigma(\mathcal{A}_0, \mathcal{R}_A)$ and act partition $\mathbb{A}_1 = (\mathbb{A}_0 \setminus \{A\}) \cup \mathcal{R}_A$.
\end{definition}

In words, the agent takes a coarse act and distinguishes two ways it might be realized.

More generally, a \emph{$k$-ary refinement} produces $\mathcal{R}_{A} = \{A \land B_1, \dots, A \land B_k\}$ where $\{B_i\}_{i=1}^k$ partition $A$. Since any $k$-ary refinement decomposes into sequential binary refinements,\footnote{Proof in the appendix.} we focus on the binary case throughout.

One can also model the addition of wholly unconsidered acts via a catch-all proposition $A_{CA} \in \mathbb{A}_0$ representing acts not yet contemplated \citep{walker2013bayesian, BalocchiFavaroNaulet}. Expansion of the act partition then becomes a special case of refinement---refinement of the catch-all.

\begin{figure}[t!]
\centering
\begin{tikzpicture}[
    rect/.style={draw, thick, minimum width=2.5cm, minimum height=2cm, anchor=south west},
    label/.style={font=\footnotesize}
]

\begin{scope}[local bounding box=left]
    \node[rect, fill=gray!20] (A) at (0,0) {};
    \node[rect, fill=white] (notA) at (2.5,0) {};
    
    \node[label] at (1.25,1) {$A$};
    \node[label] at (3.75,1) {$\neg A$};
    
    \node[below=5mm, right=23mm, font=\small] {$\mathbb{A}_0$};
\end{scope}

\draw[->, thick, shorten >=2mm, shorten <=2mm] (5.5,1) -- (6.5,1) 
    node[midway, above, font=\small] {$\mathcal{R}_A$};

\begin{scope}[xshift=7cm, local bounding box=right]
    \node[rect, fill=gray!30, minimum height=1cm] (AB1) at (0,1) {};
    \node[rect, fill=gray!10, minimum height=1cm] (AB2) at (0,0) {};
    \node[rect, fill=white] (notA2) at (2.5,0) {};
    
    \node[label] at (1.25,1.5) {$A \land B_1$};
    \node[label] at (1.25,0.5) {$A \land B_2$};
    \node[label] at (3.75,1) {$\neg A$};
    
    \node[below=5mm, right=23mm, font=\small] {$\mathbb{A}_1$};
\end{scope}
\end{tikzpicture}
\vspace*{-4pt}
\caption{A binary refinement of act $A\in\mathbb{A}_0$. The initial act partition $\mathbb{A}_0$ consists of acts $A$ and $\neg A$. Refinement produces $\mathbb{A}_1$ by splitting $A$ into the more fine-grained acts $A \land B_1$ and $A \land B_2$.}
\end{figure}

\subsubsection{Modeling Agent Attitudes Regarding Refinement}

An agent contemplating value refinement faces a decision under uncertainty: she must estimate the value of refinement without knowing what distinctions she will discover. To model this meta-uncertainty, we equip the agent with a probability space $(\Omega, \mathcal{F}, \mu_A)_A$. Each $\omega \in \Omega$ represents one possible refinement outcome for act $A$---one way the agent's credences and utilities might settle after careful consideration.

Define the refinement outcome mapping $\xi^A: \Omega \to \mathbb{R}^2 \times \Delta^2$ by $\omega \mapsto \allowbreak (u_1(\omega), u_2(\omega), \allowbreak p_1(\omega), p_2(\omega))$, where $\Delta^2 = \{(p_1, p_2) \in [0,1]^2 : p_1 + p_2 \in (0,1)\}$. Here $u_i(\omega) = U_1^\omega(A \land B_i)$ is the utility assigned to the $i$-th refined act in state $\omega$, and $p_i(\omega) = \bar{P}_1^{\omega}(A \land B_i)$ is the corresponding credence. The agent's beliefs about refinement outcomes are captured by the induced distribution:
$$\mu_A = \mu_A \circ (\xi^A)^{-1}.$$

Since refinement may lead the agent to revise $P(A)$ from $P_0(A)$ to $p_1(\omega) + p_2(\omega)$, maintaining probabilistic coherence requires rescaling. Post-refinement probabilities and utilities satisfy:
\begin{align*}
P_1^\omega(A \land B_i) &= \frac{p_i(\omega)}{Z(\omega)}, \quad U_1^\omega(A \land B_i) = u_i(\omega)
\end{align*}
where for any $A' \in \mathbb{A}_1$ with $A' \land (B_1 \lor B_2) = \bot$:
$$P_1^\omega(A') = \frac{P_0(A')}{Z(\omega)}$$
and the normalizing factor $Z(\omega) = 1 - P_0(A) + p_1(\omega) + p_2(\omega)$ ensures $\sum_{Y \in \mathbb{A}_1} P_1^\omega(Y) = 1$.

We assume $\mu_A(\{p_1 + p_2 \in (0,1)\}) = 1$; the agent does not expect value refinement to reveal that $A$ is impossible or certain. We do not require $P_1(A) = P_0(A)$; value refinement may reveal that the original act is more or less probable than initially believed.\footnote{Though, as we will see, rationality requires $\mathbb{E}[p_1 + p_2] = P_0(A)$; probability is preserved in expectation.} Moving forward, we denote the \emph{conditional probability} $q = \frac{p_1}{p_1 + p_2} = P_1(B_1 | A),$ representing the post-refinement probability of $B_1$ given $A$, which is well-defined as we assume $p_1 + p_2 > 0$ almost surely.

This framework enables ex ante reasoning about the value of value refinement. Defining $V_i(X) = P_i(X) \cdot U_i(X)$, the agent can compute:
$$\mathbb{E}_{\mu_A}\left[\max_{X \in \mathbb{A}_1} V_1(X)\right] - \max_{X \in \mathbb{A}_0} V_0(X).$$
That is, the agent can reason about the expected value of their best option after value refinement and compare it to their current best option.

\subsubsection{The Refinement Reflection Principle}

An agent must decide \textit{whether} to refine their understanding of an option. To do so, there must be some structure constraining their expectations regarding this process. We impose two conditions: a reflection principle relating current valuations to expected post-refinement valuations, and a non-triviality assumption ensuring genuine uncertainty.

\begin{definition}[Refinement Reflection Principle]
For $A \in \mathbb{A}_0$ and binary refinement $\mathcal{R}_A = \{A \land B_1, A \land B_2\}$, a rational agent's valuation obeys:
$$U_0(A) = \mathbb{E}_{\mu_A}[U_1(A)] = \mathbb{E}_{\mu_A}\left[q u_1 + (1-q) u_2 \right]$$
where $(u_1, u_2, q)\footnote{Recall that $q=\frac{p_1}{p_1+p_2}$ where $p_1+p_2>0$.} \sim \mu_A$.
\end{definition}

The refinement reflection principle (RRP) says that an agent's current valuation of $A$ equals their expected valuation after more careful consideration. It rules out the case where an agent expects that an act will be better (or worse) than the value they currently assign it. This parallels classical reflection principles where current credence equals expected future credence \citep{vanFraassen1984, GreavesWallace2006, Huttegger2013InDefenseOfReflection}.\footnote{See \cite{Dorst2024Reflection} for recent critical discussion of reflection principles.}

The RRP generalizes to $k$-ary refinements: $U_0(A) = \mathbb{E}_{\mu_A}\left[\sum_{i=1}^k q_i u_i\right]$ where $q_i = p_i / \sum_j p_j$ are the conditional probabilities $P_1(B_i | A)$.

Our second condition ensures genuine uncertainty about refinement outcomes:

\begin{assumption}[Refinement Uncertainty]
An agent has refinement uncertainty regarding an act $A$ if their credence $\sim\mu_A$ with respect to the outcomes of refinement $(u_1,u_2,p_1,p_2)$ satisfies:
\[
\mu_A\!\left( \{u_1 \neq u_2 \}\right) > 0.
\]
\end{assumption}

Refinement uncertainty says that the agent considers it possible that value refinement reveals something new. In particular, the agent thinks it \emph{possible} that the refined acts differ in value from one another, which implies that at least one may differ in value from $U_0(A)$.

Together, these conditions ensure positive probability of discovering strict improvements. If an agent has genuine uncertainty about refinement outcomes and their current valuation reflects expected post-refinement valuation, then with positive probability they identify superior options---and, being rational, select them.


\section{Value Refinement for One Agent}\label{sec:valueRefinement1A}

We now apply the framework to a single agent. Recall: an agent with decision problem $D_0 = \langle \mathcal{A}_0, \mathbb{A}_0, P_0, U_0 \rangle$ maintains uncertainty about how their act valuations decompose when refined, captured by distributions $\{\mu_A\}_{A \in \mathbb{A}_0}$ constrained by the RRP.

Before the formal result, consider why value refinement generates expected value. When an agent chooses a coarse-grained act $A$, they commit to whatever mixture of outcomes that act represents---equivalent to committing to a bundle without distinguishing its components. Value refinement unbundles the act, allowing the agent to see the finer components and choose among them. Since the agent can select the best component rather than accepting the average, expected utility improves.\footnote{This parallels option value: value refinement transforms a commitment to an average into a choice among components.} (See Figure \ref{fig:unbundling}.)

Consider an agent with act partition $\mathbb{A}_0 = \{A_1, \ldots, A_n\}$ who refines a rationalizable act $A^* \in \arg\max_{A \in \mathbb{A}_0} U_0(A)$ with respect to previously unconsidered propositions $B_1, B_2 \in \mathcal{A} \setminus \mathcal{A}_0$. Their uncertainty about refinement outcomes is captured by $\mu_{A^*}$ over $(u_1, u_2, q) \in \mathbb{R}^2 \times [0,1]$. By the RRP:
$$U_0(A^*) = \int_{\mathbb{R}^2 \times [0,1]^2} q u_1 + (1-q) u_2 \, d\mu_{A^*}(u_1, u_2, p_1, p_2)$$

Upon refinement, the agent observes a realization $(u_1, u_2, q) \sim \mu_{A^*}$, yielding post-refinement act partition $\mathbb{A}_1 = (\mathbb{A}_0 \setminus \{A^*\}) \cup \{A^* \land B_1, A^* \land B_2\}$.

\begin{figure}[t!]
\centering
\begin{tikzpicture}[
    rect/.style={draw, thick, minimum width=2.5cm, minimum height=2cm, anchor=south west},
    label/.style={font=\footnotesize}
]
\begin{scope}[local bounding box=left]
    \node[rect, fill=gray!20] (A) at (0,0) {};
    
    \node[label] at (1.25,1) {$A$};
        
    \node[above=26mm, left=-34mm, font=\small] {$U(A)=\mathbb{E}_{\mu_A}\left[qu_1+(1-q)u_2\right]$};
\end{scope}

\node at (4,2.6) {$<$};

\begin{scope}[xshift=5.5cm, local bounding box=right]
    \node[rect, fill=gray!30, minimum height=1cm] (AB) at (0,1) {};
    \node[rect, fill=gray!10, minimum height=1cm] (AnotB) at (0,0) {};
    
    \node[label] at (1.25,1.5) {$A \land B_1$};
    \node[label] at (1.25,0.5) {$A \land B_2$};
        
    \node[above=26mm, right=-5mm, font=\small] {$\mathbb{E}_{\mu_A}\left[\max \{u_1, u_2\}\right]$};
\end{scope}
\end{tikzpicture}
\vspace*{0pt}
\caption{Value refinement transforms commitment to an average of a coarse-grained bundle into the ability to select the best component among its fine-grained elements. Under RRP and refinement uncertainty, $\mathbb{E}_{\mu_A}[\max\{u_1,u_2\}] > \mathbb{E}_{\mu_A}[q u_1+(1-q)u_2]$. The expected maximum of a non-uniform bundle exceeds its expected mean.}
\label{fig:unbundling}
\end{figure}

\begin{theorem}[Value of Value Refinement]\label{thm:value-refinement}
Consider an agent with decision problem $D_0 = \langle \mathcal{A}_0, \mathbb{A}_0, P_0, U_0 \rangle$ and a rationalizable act $A^* \in \arg\max_{A \in \mathbb{A}_0} U_0(A)$. Let the agent's uncertainty about refinement outcomes be captured by $\mu_{A^*}$ over $(u_1, u_2, q) \in \mathbb{R}^2 \times (0,1)$, satisfying RRP and refinement uncertainty. Then the expected value of refinement is strictly positive:
\[
\E_{\mu_{A^*}}[\mathcal{V}_1] > \mathcal{V}_0,
\]
where $\mathcal{V}_0 = \max_{A \in \mathbb{A}_0} U_0(A)$ and $\mathcal{V}_1 = \max_{A \in \mathbb{A}_1} U_1(A)$.
\end{theorem}

Refinement uncertainty ensures the agent assigns positive probability to learning something new---that is, to $U_1(A^*) \neq U_0(A^*)$. The RRP ensures their current valuation equals their expected post-refinement valuation. Together, these imply positive probability that some refined act improves on the original. What value refinement adds is not a predictable change in average value, but finer-grained choice and the consequent ability to select above-average realizations.

The framework extends naturally to sequential value refinement:

\begin{corollary}[Monotonicity of Refinement Value]\label{cor:monotonicity}
Consider successive refinements $\mathcal{A}_0 \subset \mathcal{A}_1 \subset \cdots \subset \mathcal{A}_k$ with corresponding optimal values $\mathcal{V}_0, \mathcal{V}_1, \ldots, \mathcal{V}_k$. If each refinement satisfies the conditions of Theorem~\ref{thm:value-refinement}, then
\[
\mathcal{V}_0 < \E[\mathcal{V}_1] < \cdots < \E[\mathcal{V}_k].
\]
\end{corollary}

The preceding theorem and corollary assume refinement is costless. In practice, refinement requires time and cognitive effort. Suppose each refinement incurs cost $c > 0$, and let $\Delta_{R_i} = \mathbb{E}[\mathcal{V}_{i+1} - \mathcal{V}_i \mid \mathcal{A}_i]$ denote the expected gain from the $i$-th refinement.

\begin{definition}[Vanishing Returns]\label{def:dmr}
A sequence $\{\Delta_{R_i}\}_{i \in \mathbb{N}}$ exhibits \emph{vanishing returns} if $\Delta_{R_i} > \Delta_{R_{i+1}} \ge 0$ for all $i$, and $\lim_{i \to \infty} \Delta_{R_i} = 0$.
\end{definition}

This captures diminishing marginal value as the agent exhausts the most valuable distinctions.

\begin{theorem}[Optimal Refinement with Fixed Costs]
\label{thm:optimal-stopping}
Consider sequential refinements each costing $c > 0$. Let $\Delta_{R_i} = \mathbb{E}[\mathcal{V}_{i+1} - \mathcal{V}_i | \mathcal{F}_i]$ denote the expected marginal gain from the $i$-th refinement, and suppose $\{\Delta_{R_i}\}_{i \in \mathbb{N}}$ exhibits vanishing returns. Then:
\begin{enumerate}
    \item[(i)] If $c > \Delta_{R_0}$, the optimal policy is to never refine.
    \item[(ii)] If $c \leq \Delta_{R_0}$, the optimal stopping time is $t^* = \max\{t \in \mathbb{N} : \Delta_{R_t} \geq c\}$, and the agent performs refinements $0, 1, \ldots, t^*$, yielding net gain
    \[
    \sum_{i=0}^{t^*}(\Delta_{R_i} - c) \geq 0,
    \]
    with strict inequality when $c < \Delta_{R_0}$.
\end{enumerate}
\end{theorem}

The policy says: refine while expected marginal gain exceeds cost; stop when it no longer does. This mirrors results in optimal search and information acquisition \citep{Stigler1961, McCall1970, LippmanMcCall1976, Mortensen1986}.

Having established that value refinement increases expected utility for an individual, we turn to cases where the agent's values themselves conflict.


\section{Resolving Dilemmas without Commensuration}\label{sec:valueCommensuration1A}

Our model also applies to theories of choice that deviate from expected utility maximization. Recall Levi's \citeyearpar{levi1990hard} framework in which agents may be uncertain about how to commensurate distinct values. An agent committed to both compassion and honesty might confront actions realizing these values to different degrees: how great a gain in harm aversion justifies how great a loss in truthfulness? While ranking actions within each dimension may be straightforward, determining how to trade off between dimensions is often unclear or normatively contested.

Standard expected utility theory resolves such cases by commensuration: choose weights and maximize the aggregate. We ask whether value refinement can eliminate the need for commensuration by revealing actions that are jointly better across multiple dimensions.

Let $V_1, \ldots, V_k: \mathcal{A}' \to \mathbb{R}$ represent distinct value dimensions. Under standard aggregation, a weight vector $\mathbf{w} = (w_1, \ldots, w_k)$ with $w_i \ge 0$ and $\sum_i w_i = 1$ yields $U(A) = \sum_{i=1}^k w_i V_i(A)$. Uncertainty about $\mathbf{w}$---or the normative inadmissibility of commensuration---complicates choice when actions trade off values.

Consider the simplest case: two values $V_1, V_2$ with initial partition $\mathbb{A}_0 = \{A, \neg A\}$ where
\[
V_1(A) > V_1(\neg A) \quad \text{and} \quad V_2(A) < V_2(\neg A).
\]
The agent faces a dilemma: $A$ is favored by $V_1$, while $\neg A$ is favored by $V_2$. Resolution within expected utility theory requires commensuration---determining a weight $w \in [0,1]$ such that $U(A) = (1-w)V_1(A) + wV_2(A)$.

Value refinement opens a different path. Rather than resolving uncertainty about the correct weight, the agent refines their understanding of available actions. In certain cases, refinement eliminates the dilemma by revealing an action that excels along both dimensions, obviating commensuration entirely.

Following our framework, the agent refines $A$ by considering previously unconsidered propositions $B_1, B_2 \in \mathcal{A} \setminus \mathcal{A}_0$, producing
\[
\mathbb{A}_1 = (\mathbb{A}_0 \setminus \{A\}) \cup \{A \land B_1,\; A \land B_2\}.
\]
This refinement affects both value dimensions simultaneously. For each $V_i$, the agent maintains uncertainty about refined values, captured by distributions $\mu_A^{(i)}$ over $(v_{i,1}, v_{i,2}, q^{(i)})$ where $v_{i,j} = V_i(A \land B_j)$ and $q^{(i)} = P(B_1|A)$. The RRP extends to each dimension: $V_i(A) = \mathbb{E}_{\mu_A^{(i)}}[q^{(i)} v_{i,1} + (1-q^{(i)}) v_{i,2}]$.

\begin{definition}
An action $A^*$ \emph{multi-value dominates} a set $\mathcal{X}$ if $V_i(A^*) \ge \max_{A \in \mathcal{X}} V_i(A)$ for all $i$, with strict inequality for at least one $i$.
\end{definition}

In words, an action multi-value dominates if it is at least as good on every dimension and strictly better on at least one. When such an action exists, the agent can choose it without determining $\mathbf{w}$, since it remains optimal for any admissible weighting.

\begin{figure}[t!]
    \centering
    \begin{subfigure}[b]{0.24\textwidth}
        \centering
        \begin{tikzpicture}[scale=.7, baseline={(0,-.5)}]
            \draw[thick, ->] (0,0) -- (0,3.5);
            \draw[thick, ->] (4,0) -- (4,3.5);
            \draw[thick] (0,0) -- (4,0);
            
            \node[above] at (0,3.5) {$V_1$};
            \node[above] at (4,3.5) {$V_2$};
            
            \coordinate (y1) at (0,2.2);
            \coordinate (y2) at (4,0.5);
            \coordinate (y3) at (0,0.4);
            \coordinate (y4) at (4,2.3);
            
            \coordinate (pstar) at (2, 1.35);
            
            \draw[semithick] (y1) -- (y2);
            
            \draw[semithick]  (y3) -- (y4);
            
            \node[left] at (0,2.2) {$A$};
            \node[left] at (0,0.4) {$\neg A$};
        \end{tikzpicture}
        \caption{}
        \label{fig:base}
    \end{subfigure}
    \hfill
    \begin{subfigure}[b]{0.24\textwidth}
        \centering
        \begin{tikzpicture}[scale=.7, baseline={(0,-.5)}]
            \draw[thick, ->] (0,0) -- (0,3.5);
            \draw[thick, ->] (4,0) -- (4,3.5);
            \draw[thick] (0,0) -- (4,0);
            
            \node[above] at (0,3.5) {$V_1$};
            \node[above] at (4,3.5) {$V_2$};
            
            \coordinate (y1) at (0,2.2);
            \coordinate (y2) at (4,0.5);
            \coordinate (y3) at (0,0.4);
            \coordinate (y4) at (4,2.3);
            \coordinate (y5) at (0,2.7);
            \coordinate (y6) at (4,1.6);            
            \coordinate (y7) at (0,1.9);
            \coordinate (y8) at (4,-0.7);
            
            \coordinate (pstar) at (2, 1.35);
            
            \draw[semithick, dotted] (y1) -- (y2);
            
            \draw[semithick] (y3) -- (y4);
            
            \draw[semithick] (y5) -- (y6);
            \draw[semithick] (y7) -- (2.71,0);
            
            \node[left] at (0,2.2) {$A$};
            \node[left] at (0.1,0.4) {$\neg A$};
            \node[left] at (0.1,2.8) {$A \land B_1$};
            \node[left] at (0,1.7) {$A \land B_2$};
        \end{tikzpicture}
        \caption{}
        \label{fig:variant1}
    \end{subfigure}
    \hfill
    \begin{subfigure}[b]{0.24\textwidth}
        \centering
        \begin{tikzpicture}[scale=.7, baseline={(0,-.5)}]
            \draw[thick, ->] (0,0) -- (0,3.5);
            \draw[thick, ->] (4,0) -- (4,3.5);
            \draw[thick] (0,0) -- (4,0);
            
            \node[above] at (0,3.5) {$V_1$};
            \node[above] at (4,3.5) {$V_2$};
            
            \coordinate (y1) at (0,2.2);
            \coordinate (y2) at (4,0.5);
            \coordinate (y3) at (0,0.4);
            \coordinate (y4) at (4,2.3);
            \coordinate (y5) at (0,2.7);
            \coordinate (y6) at (4,2.7);            
            \coordinate (y7) at (0,1.9);
            \coordinate (y8) at (4,-1.5);
            
            \coordinate (pstar) at (2, 1.35);
            
            \draw[semithick, dotted] (y1) -- (y2);
            
            \draw[semithick] (y3) -- (y4);
            
            \draw[semithick] (y5) -- (y6);
            \draw[semithick] (y7) -- (1.65,0);
            
            \node[left] at (0,2.2) {$A$};
            \node[left] at (0.1,0.4) {$\neg A$};
            \node[left] at (0.1,2.8) {$A \land B_1$};
            \node[left] at (0,1.7) {$A \land B_2$};
        \end{tikzpicture}
        \caption{}
        \label{fig:variant2}
    \end{subfigure}
       \caption{How value refinement can dissolve a value conflict with two dimensions $V_1$, $V_2$ and two acts $A$, $\neg A$. The left vertical axis denotes the degree of realization of $V_1$; the right denotes $V_2$. Figure \ref{fig:base} shows the dilemma: $A$ is favored by $V_1$ and $\neg A$ by $V_2$. Figures \ref{fig:variant1} and \ref{fig:variant2} show two possible refinements: in \ref{fig:variant1} the dilemma remains, while in \ref{fig:variant2} a dominating action $A \land B_1$ is revealed, making commensuration unnecessary.}
    \label{fig:value-commensuration}
\end{figure}

\begin{theorem}[Dilemma Resolution through Value Refinement]\label{thm:dilemma}
Consider an agent with value functions $V_1, V_2: \mathcal{A}' \to \mathbb{R}$ facing a dilemma: $V_1(A) > V_1(\neg A)$ and $V_2(A) < V_2(\neg A)$. Let the agent's refinement uncertainty over $A$ be captured by a joint distribution $\mu_A$ whose support contains an open subset of the region where some refined action multi-value dominates $\{\neg A, A \land B_1, A \land B_2\}$. Then with positive probability, refinement reveals a multi-value-dominating action, resolving the dilemma without commensuration.
\end{theorem}

The support condition requires that the agent's uncertainty assigns positive probability to an open neighborhood of refinement outcomes in which domination obtains. This is a mild condition: it fails only when the agent is certain that no such domination is possible, which would require precise knowledge of how their values decompose under refinement---precisely the knowledge that value refinement is meant to provide.

Figure~\ref{fig:value-commensuration} provides geometric intuition. The vertical axes show the degree to which each action realizes $V_1$ (left) and $V_2$ (right). Figure~\ref{fig:base} depicts the initial dilemma. Figures~\ref{fig:variant1} and \ref{fig:variant2} illustrate two possible refinements: in the first, refinement provides more options but preserves the tension; in the second, refinement reveals $A \land B_1$ excelling along both dimensions, eliminating the need to determine relative weights.

Just as single-valued refinement creates option value by allowing selection of above-average components, multi-valued refinement can reveal components that simultaneously exceed the average across multiple dimensions. This outcome is not guaranteed---refinement may provide additional options without revealing a dominating action. Yet when the joint distribution includes such possibilities, value refinement bypasses rather than resolves uncertainty about value weights, providing a path between the horns of a dilemma.

Having established the value of refinement for individual agents, including its potential to resolve apparent dilemmas, we turn to strategic settings where value-of-information results break down.


\section{Value Refinement in Two-Player Zero-Sum Games}
\label{sec:valueRefinement2A0-Sum}

Value-of-information guarantees break down in game-theoretic contexts where actions influence state probabilities \citep{Hirshleifer1971, GibbardHarper1978, Skyrms1984Pragmatics}. When players' actions are interdependent, additional information can decrease expected utility by producing unfavorable opponent responses. Given this breakdown, one might expect value refinement to offer no advantage in strategic settings. We demonstrate otherwise: even in zero-sum games, value refinement yields positive expected gains at equilibrium and, when refinement is not constrained to preserve zero-sum structure, produces positive-sum opportunities.

Consider a $2\times2$ zero-sum game $G_0 = (\mathcal{N}, \mathcal{S}, U)$ with players $\mathcal{N} = \{1,2\}$, strategy sets $\mathcal{S}^1 = \{A^1, \neg A^1\}$ and $\mathcal{S}^2 = \{A^2, \neg A^2\}$, and payoffs satisfying $U^1 + U^2 \equiv 0$. This game admits a unique interior mixed Nash equilibrium.\footnote{Given the initial zero-sum structure, this is the minimax solution \citep{VonNeumannMorgenstern1944}.} Player~1 then refines action $A^1$ into $\{A^1 \land B_1, A^1 \land B_2\}$, transforming the interaction into a $3\times2$ normal form game $G_1$ which may no longer be zero-sum, with realized payoffs modeled by perturbed random variables $\epsilon^i_{jk}$ as shown in Figure~\ref{fig:refinement-transformation}.

We restrict attention to equilibrium outcomes, which requires agents' credences to satisfy common knowledge. Agents' credences $\mu_{A^1}$ encode their shared beliefs regarding the other agent's strategies and the outcomes of the refinement of the focal act $A^1$.

Given RRP, each refined act has the same \textit{ex ante} expected value as the original act. This is reflected in the fact that perturbations to agents' payoffs have mean zero: $\E_{\mu_{A^1}}[\epsilon^i_{jk}]=0$. We augment refinement uncertainty with the assumption that agents believe it is possible that their refinement payoffs are not perfectly anti-correlated---that is, there is positive probability of deviating from zero-sum structure, though deviations may produce either positive-sum or negative-sum outcomes.\footnote{If players have beliefs guaranteeing that refinement preserves zero-sum structure, our results do not hold. This is appropriate for artificial zero-sum games like chess or poker, but in markets, diplomacy, or evolutionary processes, rigid zero-sum structures are atypical.} While $G_0$ is zero-sum, a realized $G_1$ is non-zero-sum with positive probability; zero-sum holds only in expectation.

\begin{figure}[t!]
\centering
\hspace*{15pt}
\begin{minipage}{0.25\textwidth}
    \centering
    \footnotesize
    \begin{tabular}{ccc}
    \toprule
    & $A^2$ & $\neg A^2$ \\
    \midrule
    $A^1$ & $(v, -v)$ & $(\alpha, -\alpha)$ \\
    $\neg A^1$ & $(\beta, -\beta)$ & $(\gamma, -\gamma)$ \\
    \bottomrule
    \end{tabular}
    \vspace*{0.5em} \\
    \textbf{Base Game $G_0$}
\end{minipage}
\hspace*{10pt}
\begin{minipage}{0.65\textwidth}
    \centering
    \footnotesize
    \setlength{\tabcolsep}{3pt}
        \begin{tabular}{ccc}
        \toprule
        & $A^2$ & $\neg A^2$ \\
        \midrule
        $A^1\land B_1$         & $(v+\epsilon^1_{11},\, -v+\epsilon^2_{11})$ & $(\alpha+\epsilon^1_{21},\, -\alpha+\epsilon^2_{21})$ \\[3pt]
        $A^1 \land B_2$   & $(v+\epsilon^1_{12},\, -v+\epsilon^2_{12})$ & $(\alpha+\epsilon^1_{22},\, -\alpha+\epsilon^2_{22})$ \\[3pt]
        $\neg A^1$           & $(\beta,\, -\beta)$                        & $(\gamma,\, -\gamma)$ \\
        \bottomrule
        \end{tabular}
    \vspace*{0.5em} \\
    \textbf{Refined Game $G_1$}    
\end{minipage}
\caption{Transformation of a $2 \times 2$ zero-sum game into a $3 \times 2$ game through value refinement. The row player refines $A^1$ into $\{A^1 \land B_1, A^1 \land B_2\}$; payoffs are perturbed by noise terms $\epsilon^i_{jk}$.}
\label{fig:refinement-transformation}
\end{figure}

We are interested in aggregate welfare in the refined game. Denote the Nash equilibria of game $G_k$ by $\mathrm{NE}(G_k)$ for $k \in \{0,1\}$, and define the welfare-optimal equilibrium payoff as $W^*_k = \max_{\sigma \in \mathrm{NE}(G_k)} [U_1^k(\sigma) + U_2^k(\sigma)]$. Since $G_0$ is zero-sum, $W^*_0 = 0$. We show that value refinement yields strict expected improvement.

\begin{theorem}[Zero-Sum Escape from Unilateral Value Refinement]\label{thm:zero-sum}
Consider a $2 \times 2$ zero-sum normal form game $G_0 = (\mathcal{N}, \mathcal{S}, U)$ with players $\mathcal{N} = \{1,2\}$, strategy sets $\mathcal{S}^1 = \{A^1, \neg A^1\}$ and $\mathcal{S}^2 = \{A^2, \neg A^2\}$. Let either player unilaterally refine one of their acts $A$ to $\{A \land B_1, A \land B_2\}$ and let both players' credences $\mu_A$ regarding refinement satisfy RRP, refinement uncertainty, and the belief that payoff perturbations are not perfectly anti-correlated. Then the expected welfare-optimal equilibrium strictly increases after refinement:
\[
\mathbb{E}[W^*_1] > W^*_0.
\]
\end{theorem}

The mechanism is option value plus asymmetric filtering. Value refinement splits an act into two variants whose expected values, by RRP, equal that of the original. But equilibrium behavior responds asymmetrically: when both players agree that one refined variant dominates the other, the inferior variant is excluded from equilibrium play. This agreement occurs precisely when the perturbations favor the same variant for both players---an event with positive probability under non-adversarial refinement. On such realizations, equilibrium selects the mutually preferred variant, whose welfare exceeds zero because both players' perturbations are positive. On realizations where players disagree, neither variant dominates, and equilibrium welfare remains zero in expectation. The asymmetry---filtering out bad realizations while retaining good ones---generates strictly positive expected welfare.

Apparent zero-sum conflicts are fragile under fine-graining. Unilateral value refinement can reveal options that both players prefer, transforming perceived opposition into latent coordination. Because realized games need not preserve zero-sum structure, these gains manifest as positive-sum equilibria absent from the coarser representation of the interaction.


\section{Value Refinement in Nash Bargaining}
\label{sec:valueRefinementNashBargaining}

We now analyze value refinement in Nash bargaining \citep{Nash1950}. When a one-dimensional resource split is refined, separable dimensions become explicit and independently allocable, generating a larger feasible set. We show that bargaining solutions in the larger set Pareto-dominate in expectation when agents weight the dimensions differently, with gains increasing as preferences become more complementary.

Consider a Nash bargaining game $\mathcal{G}_0 = (\mathcal{N}, \mathcal{X}_0, (U^1_0, U^2_0), \mathbb{A}_0, d)$ where $\mathcal{N} = \{1, 2\}$, allocation space $\mathcal{X}_0 = [0,1]$ gives Agent 1 fraction $x$ and Agent 2 fraction $1-x$, and $U^i_0: [0,1] \to \mathbb{R}$ are continuous, increasing, and concave utility functions. The feasible utility set is
\[
\mathbb{A}_0 = \{(U^1_0(x), U^2_0(1-x)) : x \in [0,1]\} \subset \mathbb{R}^2
\]
with disagreement point $d = (d_1, d_2)$. We use $\mathbb{A}$ to emphasize continuity with act partitions: value refinement here operates as simultaneous binary refinement on every element of the feasible set.\footnote{The standard assumptions ensure $\mathbb{A}_0$ is compact, convex, and comprehensive above $d$---i.e., $(u_1, u_2) \in \mathbb{A}_0$ and $d \le u' \le u$ implies $u' \in \mathbb{A}_0$.}

The Nash bargaining solution uniquely satisfies Pareto optimality, symmetry, independence of irrelevant alternatives, and invariance to affine utility transformations.\footnote{These axioms jointly imply $f^{\mathrm{NB}}(\mathbb{A}, d) = \arg\max_{v \in \mathbb{A},\, v \ge d} (v_1 - d_1)(v_2 - d_2)$. Alternative solution concepts include Kalai-Smorodinsky (proportional gains), egalitarian (equal gains), and utilitarian (maximal sum).} On $(\mathbb{A}_0, d)$:
\[
f^{\mathrm{NB}}(\mathbb{A}_0, d) = \arg\max_{(u_1, u_2) \in \mathbb{A}_0} (u_1 - d_1)(u_2 - d_2)
\]
yielding allocation $x^*$ with utilities $U^{1*}_0$ and $U^{2*}_0$.

Value refinement transforms $\mathcal{G}_0$ into $\mathcal{G}_1(\omega) = (\mathcal{N}, \mathcal{X}_1, (U^1(\cdot, \omega), U^2(\cdot, \omega)), \mathbb{A}_1(\omega), d)$. The allocation space expands from one to two dimensions: $\mathcal{X}_1 = [0,1]^2$, where Agent 1 receives $(x_1, x_2)$ and Agent 2 receives $(1-x_1, 1-x_2)$. Utilities become additively separable:
\[
U^i(x_1, x_2, \omega) = w^i_1(\omega) \cdot v_1(x_1) + w^i_2(\omega) \cdot v_2(x_2),
\]
where $v_j: [0,1] \to \mathbb{R}_+$ are common value functions for each dimension and $w^i_1(\omega) + w^i_2(\omega) = 1$ are agent-specific weights. Pre-refinement allocations embed as \emph{bundled} allocations where $x_1 = x_2 = x$.

Let $(\Omega, \Sigma, \mu)$ be a probability space over refinement realizations satisfying refinement uncertainty. RRP requires that pre-refinement utilities equal expected utilities under forced bundling:
\[
U^i_0(x) = \E_\mu[U^i(x, x, \omega)] = \E[w^i_1] \cdot v_1(x) + \E[w^i_2] \cdot v_2(x).
\]
When $\E[w^i_1] = 1/2$ and the dimensions have equal value ($v_1 = v_2 = v$), this reduces to $U^i_0(x) = v(x)$, so both agents share the same pre-refinement utility function.

Refinement uncertainty here consists in the condition that agents assign positive probability to having distinct weights over dimensions of value: $\mu(w^1_j \neq w^2_j) > 0$ for $j \in \{0,1\}$.

\begin{figure}[t!]
\centering
\begin{tikzpicture}[scale=0.8]
\begin{scope}[shift={(-6,0)}]
\draw[->] (0,0) -- (4.5,0) node[right] {$U^1$};
\draw[->] (0,0) -- (0,4.5) node[above] {$U^2$};
\draw[ultra thick, gray] (0,4) -- (4,0);
\fill[gray!20, opacity=0.5] (0,0) -- (0,4) -- (4,0) -- cycle;
\node[circle, fill=gray, inner sep=2pt] at (0,4) {};
\node[circle, fill=gray, inner sep=2pt] at (4,0) {};
\node[star, star points=5, star point ratio=2, fill=black, inner sep=2pt] at (2,2) {};
\node[below left] at (0,0) {$(0,0)$};
\node[left] at (0,4) {$(0,1)$};
\node[below] at (4,0) {$(1,0)$};
\node[above right] at (2,2) {$(\tfrac{1}{2},\tfrac{1}{2})$};
\node[below] at (2,-0.8) {\textbf{Pre-refinement}};
\end{scope}
\begin{scope}[shift={(1,0)}]
\draw[->] (0,0) -- (4.5,0) node[right] {$U^1$};
\draw[->] (0,0) -- (0,4.5) node[above] {$U^2$};
\fill[gray!20, opacity=0.3] (0,0) -- (0,4) -- (4,4) -- (4,0) -- cycle;
\draw[ultra thick, gray!70!black] (0,4) -- (4,4) -- (4,0);
\draw[ultra thick, gray!70!black] (0,0) -- (0,4);
\draw[ultra thick, gray!70!black] (0,0) -- (4,0);
\draw[ultra thick, gray, dashed] (0,4) -- (4,0);
\node[circle, fill=gray, inner sep=2pt] at (0,0) {};
\node[circle, fill=gray, inner sep=2pt] at (0,4) {};
\node[circle, fill=gray, inner sep=2pt] at (4,0) {};
\node[star, star points=5, star point ratio=2, fill=black, inner sep=2pt] at (4,4) {};
\node[below left] at (0,0) {$(0,0)$};
\node[left] at (0,4) {$(0,1)$};
\node[below] at (4,0) {$(1,0)$};
\node[above right] at (4,4) {$(1,1)$};
\node[below] at (2,-0.8) {\textbf{Post-refinement}};
\end{scope}
\end{tikzpicture}
\caption{Value refinement expands the feasible set from a line (left) to a rectangle (right) when preferences are orthogonal. The dashed line shows bundled allocations; the full shaded region includes allocations where dimensions are allocated independently. When $\theta = \pi/2$, the Nash solution moves from $(\tfrac{1}{2}, \tfrac{1}{2})$ to $(1,1)$, doubling both agents' payoffs.}
\label{fig:nash_refinement}
\end{figure}

\begin{theorem}[Value of Refinement in Nash Bargaining]\label{thm:nash}
Consider a symmetric Nash bargaining game $\mathcal{G}_0 = (\mathcal{N}, \mathcal{X}_0, (U^1_0, U^2_0), \mathbb{A}_0, d)$ with players $\mathcal{N} = \{1,2\}$, allocation space $\mathcal{X}_0 = [0,1]$, disagreement point $d = (d_1, d_2)$ with $d_1 = d_2$, and symmetric preferences $U^1_0(x) = U^2_0(1-x)$. Refinement transforms the game into $\mathcal{G}_1(\omega) = (\mathcal{N}, \mathcal{X}_1, (U^1, U^2), \mathbb{A}_1(\omega), d)$ with allocation space $\mathcal{X}_1 = [0,1]^2$, value functions $v_1, v_2: [0,1] \to \mathbb{R}_+$ (continuous, strictly increasing, strictly concave), and additively separable utilities
\[
U^i(x_1, x_2, \omega) = w^i_1(\omega) v_1(x_1) + w^i_2(\omega) v_2(x_2), \quad w^i_1 + w^i_2 = 1.
\]
Let the agents' credences $\mu$ satisfy RRP and refinement uncertainty: $\mu(w^1_1 \neq w^2_1) > 0$. Then the expected Nash payoffs post-refinement strictly Pareto-dominate pre-refinement payoffs:
\[
\E[U^{i*}_1] > U^{i*}_0 \quad \text{for } i \in \{1,2\}.
\]
\end{theorem}

The mechanism is comparative advantage through dimensional specialization. When agents weight dimensions differently, the bundled allocation---where both dimensions are split identically---is Pareto inefficient. To see why, consider perturbations from the symmetric bundled allocation. Agent 1 benefits from reallocations that increase her weighted combination of dimensional gains; Agent 2 benefits from reallocations in the opposite direction, since she receives what Agent 1 does not. These requirements are geometrically compatible---a direction exists that benefits both---exactly when the agents' weight vectors $(w^1_1, w^1_2)$ and $(w^2_1, w^2_2)$ differ. Since both vectors have nonnegative components summing to one, they differ if and only if $w^1_1 \neq w^2_1$. When this holds, Pareto improvements over the bundled Nash solution exist: give Agent 1 more of dimension 1 (which she values relatively more) while giving Agent 2 more of dimension 2 (which he values relatively more).

The Nash bargaining solution responds to this expansion via expansion monotonicity: when the feasible set grows while the disagreement point remains fixed, each agent's Nash payoff weakly increases---and strictly increases when the expansion contains Pareto improvements over the original solution. Whenever agents discover they weight dimensions differently, both capture strictly higher payoffs than under forced bundling. Refinement uncertainty ensures this scenario occurs with positive probability, so in expectation both agents' post-refinement payoffs strictly exceed their pre-refinement payoffs.

\begin{corollary}[Gains Increase with Preference Divergence]\label{cor:complementarity}
In the setting of Theorem~\ref{thm:nash}, fix marginal distributions for $w^1_1$ and $w^2_1$ with $\E[w^i_1] = 1/2$ and $\Var(w^i_1) = \sigma^2 > 0$. The expected benefit of refinement $\E[U^{i*}_1] - U^{i*}_0$ is strictly decreasing in the correlation $\Corr(w^1_1, w^2_1)$, with the greatest gains occurring when preferences are perfectly negatively correlated: $w^2_1 = 1 - w^1_1$ almost surely.
\end{corollary}

When agents' weights are positively correlated, they expect to value the same dimensions highly, limiting opportunities for mutually beneficial trade. When weights are less correlated---meaning agents tend to want different things---the scope for gains through specialization expands. The benefit of refinement depends on how different the agents' preferences turn out to be, measured by $|\Delta| = |w^1_1 - w^2_1|$. Lower covariance spreads probability mass toward larger preference divergence, amplifying expected gains. The maximum occurs with perfectly negatively correlated preferences, where Agent~1 values dimension~1 exactly as much as Agent~2 values dimension~2. At this extreme, each dimension can be allocated entirely to the agent who values it more, fully exploiting specialization.

Figure~\ref{fig:nash_refinement} illustrates such a case of perfectly complementary preferences, as in the vignette of the Arborist and the Baker in \S \ref{sec:Vignette}, where the feasible set expands from a line to a rectangle and both agents' payoffs double.

In bargaining, agents benefit in expectation from refining their understanding of what they are negotiating over. Value refinement reveals dimensional structure, transforming one-dimensional division into opportunities for mutually beneficial trade. The result connects to insights about gains from trade \citep{Ricardo1817, Edgeworth1881, ArrowDebreu1954}: agents with different preferences can both benefit by specializing in what they value most, and value refinement is the process by which such opportunities are discovered.


\section{Implications and Future Directions}\label{sec:Discussion}

Under broad conditions, value refinement at reasonable cost is \emph{ex ante} beneficial for a rational agent (Theorem \ref{thm:value-refinement}), can resolve dilemmas involving incommensurable values without commensuration (Theorem \ref{thm:dilemma}), turns zero-sum conflicts into positive-sum opportunities (Theorem \ref{thm:zero-sum}), and expands Nash-bargaining frontiers so that each party secures strictly higher expected payoff (Theorem \ref{thm:nash}).

Unlike idealized agents, real decision makers---individuals, AI systems, institutions---face uncertainty about whether their current representation of options and valuations is adequate. One lesson of this work is that when conflict arises, whether from individual dilemmas or coordination problems, there is value in resisting the impulse to act on one's current representation of the situation. Taking conflict as an occasion to reflect on one's options and values can avert or mitigate apparent incompatibilities; failing to do so can entrench suboptimal perspectives that do not reflect the agent's ultimate interests.

The results extend decision theory in a specific direction. Savage and von Neumann-Morgenstern treat the act-state-outcome division as fixed; the atomless Boolean algebra of the Jeffrey-Bolker framework permits indefinite refinement within a single model. We provide a natural formalism for this process. Coherence arguments---Dutch books, accuracy dominance, sequential optimality---assume a settled utility function. Our results show that taking settlement for granted is itself a substantive decision: foregoing value refinement is strictly dominated by refining first when the conditions of our theorems hold. Rational choice therefore has a two-stage character: \emph{refine, then maximize}.

This two-stage structure generates a tension we might call the paradox of bounded optimization: an agent who directly maximizes their current utility function is strictly dominated by one who first refines the values from which that function derives. The tension echoes the paradox of hedonism---just as pursuing pleasure directly can preclude achieving it, maximizing a coarse utility function can preclude outcomes that better realize the agent's underlying values. Resolution requires treating act partitions and utilities as provisional: hypotheses about one's options and values, open to disciplined revision.

This analysis favors pluralistic over monistic approaches to value. Monistic frameworks, where a single utility function governs choice, resist incoherence but can fail to capture an agent's uncertainty about their underlying values. Value pluralism, by contrast, treats multiple irreducible values as legitimate inputs to deliberation. Our results show that pluralism need not entail decision paralysis: finer descriptions of available acts can expose alternatives that dominate on every value dimension. Consulting diverse values before committing to tradeoffs improves outcomes.

Value-of-information theorems guarantee that learning about the world provides non-negative expected value. We provide the parallel guarantee for learning about one's values. Agents who treat their utility functions as provisional hypotheses fare better individually and collectively. Rationality is not merely calculation with a fixed utility function; it is the iterative work of making that function worthy of calculation.

We do not claim completeness. Our results are formulated within the Jeffrey-Bolker framework; analogous results in the frameworks of Savage, von Neumann-Morgenstern, or Anscombe-Aumann remain to be developed. We assumed negligible cognitive and temporal costs; introducing stochastic refinement costs would permit optimal-stopping analysis of how much reflection suffices under various conditions. Our strategic results rely on common knowledge; relaxing this assumption may attenuate positive-sum guarantees, and bounding the resulting effect sizes remains open. In iterated interactions, early-round refinement shapes later payoffs and information flows; richer dynamics---including signaling about value refinement intentions---await investigation.


\appendix
\section{Mathematical Appendix}

\begin{proposition}[$k$-ary Refinements via Binary Refinements]
\label{prop:k-ary}
Any $k$-ary refinement of $A \in \mathbb{A}_0$ can be achieved through $k-1$ binary refinements.
\end{proposition}

\begin{proof}
Let $\{B_i\}_{i=1}^k$ partition $A$. We proceed by induction on the number of refinement steps. At step $j \in \{1, \ldots, k-1\}$, the current partition of $A$ consists of the singletons $A \land B_1, \ldots, A \land B_{j-1}$ together with the residual element $R_j = A \land (B_j \lor \cdots \lor B_k)$. We perform a binary refinement of $R_j$ into $A \land B_j$ and $R_{j+1} = A \land (B_{j+1} \lor \cdots \lor B_k)$. 

For the base case $j = 1$, we have $R_1 = A$, and the binary refinement produces $\{A \land B_1, R_2\}$. For the inductive step, assume after step $j-1$ the partition is $\{A \land B_1, \ldots, A \land B_{j-1}, R_j\}$. After step $j$, it becomes $\{A \land B_1, \ldots, A \land B_j, R_{j+1}\}$. After step $k-1$, the residual is $R_k = A \land B_k$, yielding the desired partition $\{A \land B_1, \ldots, A \land B_k\}$. \qedhere
\end{proof}

\begin{proof}[Proof of Theorem~\ref*{thm:value-refinement} (Value of Value Refinement)]
Consider an agent with decision problem $D_0 = \langle \mathcal{A}_0, \mathbb{A}_0, P_0, U_0 \rangle$ who refines a rationalizable act $A^* \in \arg\max_{A \in \mathbb{A}_0} U_0(A)$ into two finer acts $\{A^* \land B_1, A^* \land B_2\}$. The agent's uncertainty about post-refinement utilities and probabilities is captured by a distribution $\mu_{A^*}$ over $(u_1, u_2, p_1, p_2)$, where $u_i = U_1(A^* \land B_i)$ and $p_i = P_1(A^* \land B_i)$. We require three conditions: refinement uncertainty, meaning $\mu_{A^*}(u_1 \neq u_2) > 0$; non-degeneracy, so that $p_1, p_2 > 0$ almost surely; and the Refinement Reflection Principle, which ensures $U_0(A^*) = \E[(p_1 u_1 + p_2 u_2)/(p_1 + p_2)]$.

Let $\mathcal{V}_0 = U_0(A^*)$ denote the pre-refinement value and $\mathcal{V}_1 = \max_{A \in \mathbb{A}_1} U_1(A)$ the post-refinement value. Since the refined partition includes $A^* \land B_1$ and $A^* \land B_2$, we have $\mathcal{V}_1 \geq \max\{u_1, u_2\}$.

It suffices to show $\E[\max\{u_1, u_2\}] > \E[q u_1 + (1-q) u_2]$ where $q = p_1/(p_1 + p_2) \in (0,1)$. For any $q \in (0,1)$ and any $u_1, u_2 \in \R$, the inequality $\max\{u_1, u_2\} \geq q u_1 + (1-q) u_2$ holds, with equality if and only if $u_1 = u_2$. This is because the maximum of two numbers weakly exceeds any convex combination thereof, with strict inequality when they differ. By refinement uncertainty, $\mu_{A^*}(u_1 \neq u_2) > 0$, so the inequality is strict on a set of positive measure. Taking expectations yields
\[
\E[\max\{u_1, u_2\}] > \E[q u_1 + (1-q) u_2] = \E[U_1(A^*)] = U_0(A^*) = \mathcal{V}_0,
\]
where the second equality uses RRP and the third uses rationalizability of $A^*$. Since $\E[\mathcal{V}_1] \geq \E[\max\{u_1, u_2\}]$, we conclude $\E[\mathcal{V}_1] > \mathcal{V}_0$. \qedhere
\end{proof}

\begin{proof}[Proof of Corollary~\ref*{cor:monotonicity} (Monotonicity of Refinement Value)]
Consider successive refinements producing increasingly fine algebras $\mathcal{A}_0 \subset \mathcal{A}_1 \subset \cdots \subset \mathcal{A}_k$, with corresponding optimal values $\mathcal{V}_0, \mathcal{V}_1, \ldots, \mathcal{V}_k$. If each refinement from $\mathcal{A}_{j-1}$ to $\mathcal{A}_j$ satisfies the conditions of Theorem~\ref*{thm:value-refinement}---namely RRP, refinement uncertainty, and non-degeneracy---then applying that theorem iteratively yields $\E[\mathcal{V}_{j} \mid \mathcal{A}_{j-1}] > \mathcal{V}_{j-1}$ at each step. By the law of iterated expectations, $\E[\mathcal{V}_j] > \E[\mathcal{V}_{j-1}]$, yielding the chain $\mathcal{V}_0 < \E[\mathcal{V}_1] < \cdots < \E[\mathcal{V}_k]$. \qedhere
\end{proof}

\begin{proof}[Proof of Theorem~\ref*{thm:optimal-stopping} (Optimal Refinement with Fixed Costs)]
Suppose each refinement costs $c > 0$ and the $i$-th refinement yields expected gain $\Delta_{R_i} = \E[\mathcal{V}_{i+1} - \mathcal{V}_i \mid \mathcal{A}_i]$. The sequence $\{\Delta_{R_i}\}$ exhibits vanishing returns: $\Delta_{R_i} > \Delta_{R_{i+1}} \geq 0$ for all $i$, and $\lim_{i \to \infty} \Delta_{R_i} = 0$.

For case (i), suppose $c > \Delta_{R_0}$. Since $\{\Delta_{R_i}\}$ is strictly decreasing, $\Delta_{R_i} < \Delta_{R_0} < c$ for all $i \geq 0$. Every refinement has negative expected net value, so the optimal policy is to never refine.

For case (ii), suppose $c \leq \Delta_{R_0}$. The set $\{t \in \mathbb{N} : \Delta_{R_t} \geq c\}$ is nonempty (containing $0$) and bounded above: since $\Delta_{R_i} \to 0$ and $c > 0$, there exists $T$ such that $\Delta_{R_t} < c$ for all $t > T$. Hence $t^* = \max\{t : \Delta_{R_t} \geq c\}$ exists and is finite.

We verify optimality by backward induction. For any $t > t^*$, we have $\Delta_{R_t} < c$, so the expected net gain from refinement $t$ is negative; stopping is strictly optimal. For $t = t^*$, we have $\Delta_{R_{t^*}} \geq c$, so refinement $t^*$ yields non-negative expected net gain, and stopping thereafter is optimal by the previous argument. For $t < t^*$, the strict decrease of $\{\Delta_{R_i}\}$ implies $\Delta_{R_t} > \Delta_{R_{t^*}} \geq c$, so refinement $t$ is strictly beneficial.

The total net gain is $\sum_{i=0}^{t^*}(\Delta_{R_i} - c)$. When $c < \Delta_{R_0}$, each summand with $i < t^*$ satisfies $\Delta_{R_i} > c$ by strict decrease and $\Delta_{R_{t^*}} \geq c$, so the sum is strictly positive. When $c = \Delta_{R_0}$ and $\Delta_{R_1} < c$, we have $t^* = 0$ and net gain $\Delta_{R_0} - c = 0$. \qedhere
\end{proof}

\begin{proof}[Proof of Theorem~\ref*{thm:dilemma} (Dilemma Resolution through Value Refinement)]
Consider an agent with value functions $V_1, V_2: \mathcal{A}' \to \R$ facing a dilemma: $V_1(A) > V_1(\neg A)$ while $V_2(A) < V_2(\neg A)$. The agent refines $A$ into $\{A \land B_1, A \land B_2\}$, with post-refinement values $v_{i,j} = V_i(A \land B_j)$ drawn from a joint distribution $\mu_A$.

Define the events
\begin{align*}
\Omega_1 &= \{v_{1,1} > \max\{V_1(\neg A), v_{1,2}\} \text{ and } v_{2,1} > \max\{V_2(\neg A), v_{2,2}\}\}, \\
\Omega_2 &= \{v_{1,2} > \max\{V_1(\neg A), v_{1,1}\} \text{ and } v_{2,2} > \max\{V_2(\neg A), v_{2,1}\}\}.
\end{align*}
On $\Omega_1$, action $A \land B_1$ multi-value dominates; on $\Omega_2$, action $A \land B_2$ does. These events are disjoint.

By hypothesis, the support of $\mu_A$ contains an open set $\mathcal{O} \subseteq \Omega_1 \cup \Omega_2$. Since any distribution assigns positive measure to open sets in its support, $\mu_A(\Omega_1 \cup \Omega_2) \geq \mu_A(\mathcal{O}) > 0$.

On $\Omega_1 \cup \Omega_2$, let $A^*$ denote the dominant action. For any weighting $w \in [0,1]$, we have
\[
U_w(A^*) = (1-w)V_1(A^*) + wV_2(A^*) > (1-w)V_1(X) + wV_2(X) = U_w(X)
\]
for all other actions $X$. The agent can choose $A^*$ without determining $w$, resolving the dilemma through value refinement rather than commensuration. \qedhere
\end{proof}

\begin{proof}[Proof of Theorem~\ref*{thm:zero-sum} (Zero-Sum Escape from Unilateral Value Refinement)]
Consider a two-player zero-sum game $G_0 = (\mathcal{N}, \mathcal{S}, U)$ with $\mathcal{N} = \{1,2\}$, strategy sets $\mathcal{S}^1 = \{A, \neg A\}$ and $\mathcal{S}^2 = \{B, \neg B\}$, and payoffs satisfying $U^1 + U^2 \equiv 0$. Let Player~1 refine $A$ into $\{A \land B_1, A \land B_2\}$, producing the refined game $G_\epsilon$ with payoff perturbations $\epsilon_{ij}$ for Player~1 and $\delta_{ij}$ for Player~2 at profile $(A \land B_i, j)$. We assume: (i) RRP, so $\E[\epsilon_{ij}] = \E[\delta_{ij}] = 0$; (ii) refinement uncertainty for both players, so $\mu(\epsilon_{1j} \neq \epsilon_{2j}) > 0$ and $\mu(\delta_{1j} \neq \delta_{2j}) > 0$ for some $j$; and (iii) non-adversarial refinement, so $\Cov(\epsilon_{ij}, \delta_{ij}) \geq 0$.

Define welfare $Z_{ij} = \epsilon_{ij} + \delta_{ij}$ at profile $(A \land B_i, j)$ and welfare zero at $(\neg A, j)$. By RRP, $\E[Z_{ij}] = 0$. Let $W^*_\epsilon$ denote the welfare at the welfare-optimal Nash equilibrium of $G_\epsilon$.

Let $\mathcal{A}$ denote the event of full agreement: both players prefer $B_1$ to $B_2$ against both opponent actions, or vice versa. On $\mathcal{A}$, one refined action---call it $B^*$---strictly dominates the other for Player~1 and is also preferred by Player~2; the dominated action is excluded from all equilibria.

Partition $\mathcal{A}$ into $E_0$, where $\neg A$ dominates $B^*$ so equilibrium welfare is zero, and $E_1 = \mathcal{A} \setminus E_0$, where $B^*$ is preferred to $\neg A$ against at least one opponent action. On $E_1$, both players prefer $B^*$ among the refined actions, so $\epsilon^*_j > \epsilon^{\text{other}}_j$ and $\delta^*_j > \delta^{\text{other}}_j$ for the relevant $j$. By exchangeability, conditioning on being the larger of two mean-zero exchangeable variables yields a positive expectation. Hence $\E[\epsilon^*_j \mid E_1] > 0$ and $\E[\delta^*_j \mid E_1] > 0$, so $\E[W^*_\epsilon \mid E_1] > 0$.

On $\mathcal{A}^c$, neither refined action dominates. We show $\E[Z_{ij} \mid \mathcal{A}^c] = 0$. By exchangeability, $\epsilon_1 + \epsilon_2$ is symmetric under index-swapping. The events $\{\epsilon_1 > \epsilon_2\}$ and $\{\epsilon_2 > \epsilon_1\}$ have equal probability and are symmetric images; hence $\E[\epsilon_1 + \epsilon_2 \mid \epsilon_1 > \epsilon_2] = \E[\epsilon_1 + \epsilon_2 \mid \epsilon_2 > \epsilon_1]$. Since these partition the space (up to measure zero) and $\E[\epsilon_1 + \epsilon_2] = 0$, each conditional expectation is zero. The same holds for $\delta$, so $\E[Z_{ij} \mid \mathcal{A}^c] = 0$. Since equilibrium on $\mathcal{A}^c$ selects among $\{Z_1, Z_2, 0\}$, we have $\E[W^*_\epsilon \mid \mathcal{A}^c] \geq 0$.

Combining, $\E[W^*_\epsilon] = \Pr(E_1)\E[W^*_\epsilon \mid E_1] + \Pr(E_0) \cdot 0 + \Pr(\mathcal{A}^c) \cdot 0 > 0$, where the strict inequality uses $\Pr(E_1) > 0$ from refinement uncertainty. \qedhere
\end{proof}

\begin{proof}[Proof of Theorem~\ref*{thm:nash} (Value of Refinement in Nash Bargaining)]
Consider a symmetric Nash bargaining game $\mathcal{G}_0 = (\mathcal{N}, \mathcal{X}_0, (U^1_0, U^2_0), \mathbb{A}_0, d)$ with $\mathcal{N} = \{1,2\}$, allocation space $\mathcal{X}_0 = [0,1]$, and disagreement point $d = (d_1, d_2)$ with $d_1 = d_2$. By symmetry, $U^1_0(x) = U^2_0(1-x) = u(x)$ for some continuous, strictly increasing, strictly concave function $u$ with $u(0) = 0$. Refinement transforms the game into $\mathcal{G}_1(\omega) = (\mathcal{N}, \mathcal{X}_1, (U^1, U^2), \mathbb{A}_1(\omega), d)$ with allocation space $\mathcal{X}_1 = [0,1]^2$ and additively separable utilities
\[
U^1(x_1, x_2, \omega) = w^1_1 v_1(x_1) + w^1_2 v_2(x_2), \quad U^2(x_1, x_2, \omega) = w^2_1 v_1(1-x_1) + w^2_2 v_2(1-x_2),
\]
where $v_1, v_2: [0,1] \to \mathbb{R}_+$ are continuous, strictly increasing, strictly concave with $v_j(0) = 0$, and $w^i_1(\omega) + w^i_2(\omega) = 1$.

By symmetry and strict concavity of $u$, the Nash product $u(x) \cdot u(1-x)$ is uniquely maximized at $x^* = 1/2$, yielding $U^{1*}_0 = U^{2*}_0 = u(1/2) \equiv U_0$.

At bundled allocations $(x, x)$, Agent~1's expected utility is $\E[w^1_1] v_1(x) + \E[w^1_2] v_2(x)$. By RRP and symmetric priors, $\E[w^i_j] = 1/2$, so $\E[U^1(x, x, \omega)] = \tfrac{1}{2}v_1(x) + \tfrac{1}{2}v_2(x) = u(x)$; pre-refinement utilities are preserved in expectation under bundling.

We now show that when $w^1_1 \neq w^2_1$, Pareto improvements over the bundled allocation exist. Fix $\omega$ with $w^1_1 \neq w^2_1$ and consider perturbations from $(1/2, 1/2)$ to $(1/2 + \delta_1, 1/2 + \delta_2)$. The first-order utility changes are
\begin{align*}
\Delta U^1 &= w^1_1 v'_1(1/2)\, \delta_1 + w^1_2 v'_2(1/2)\, \delta_2, \\
\Delta U^2 &= -w^2_1 v'_1(1/2)\, \delta_1 - w^2_2 v'_2(1/2)\, \delta_2.
\end{align*}
The conditions $\Delta U^1 > 0$ and $\Delta U^2 > 0$ define half-spaces with normal vectors proportional to $(w^1_1, w^1_2)$ and $(w^2_1, w^2_2)$ after scaling by the positive derivatives $v'_j(1/2)$. These half-spaces intersect if and only if the normals are not proportional, which holds exactly when $w^1_1 \neq w^2_1$ since both weight vectors sum to unity. When $w^1_1 > w^2_1$, the improving direction has $\delta_1 > 0$ and $\delta_2 < 0$: Agent~1 gains in dimension~1 while Agent~2 gains in dimension~2.

The feasible set expands from $\mathbb{A}_0$ to $\mathbb{A}_1(\omega) \supseteq \mathbb{A}_0$. Nash bargaining satisfies expansion monotonicity: when the feasible set expands with disagreement point fixed, each agent's payoff weakly increases, and strictly increases when the expansion contains Pareto improvements over the original solution. By the preceding analysis, when $w^1_1 \neq w^2_1$, such improvements exist, so $U^{i*}_1(\omega) > U_0$. When $w^1_1 = w^2_1$, the bundled allocation remains optimal and $U^{i*}_1(\omega) = U_0$.

By refinement uncertainty, $\mu(w^1_1 \neq w^2_1) > 0$. Decomposing expectations,
\[
\E[U^{i*}_1] = U_0 \cdot \Pr(w^1_1 = w^2_1) + \E[U^{i*}_1 \mid w^1_1 \neq w^2_1] \cdot \Pr(w^1_1 \neq w^2_1).
\]
Since $\E[U^{i*}_1 \mid w^1_1 \neq w^2_1] > U_0$ and $\Pr(w^1_1 \neq w^2_1) > 0$, we conclude $\E[U^{i*}_1] > U_0 = U^{i*}_0$. \qedhere
\end{proof}

\begin{proof}[Proof of Corollary~\ref*{cor:complementarity} (Gains Increase with Preference Divergence)]
Consider a symmetric Nash bargaining game $\mathcal{G}_0$ satisfying the hypotheses of Theorem~\ref{thm:nash}, with pre-refinement Nash payoff $U^{i*}_0 = u(1/2)$ for both agents. Refinement reveals two dimensions with value functions $v_1, v_2: [0,1] \to \mathbb{R}_+$ (continuous, strictly increasing, strictly concave) and additively separable utilities $U^i = w^i_1 v_1 + w^i_2 v_2$ where $w^i_1 + w^i_2 = 1$. Fix marginal distributions for $w^1_1$ and $w^2_1$ with $\E[w^i_1] = 1/2$ and $\Var(w^i_1) = \sigma^2 > 0$, and let $\gamma = \Cov(w^1_1, w^2_1)$.

For each realization $\omega$, let $G^i(\omega) = U^{i*}_1(\omega) - U^{i*}_0$ denote agent $i$'s gain from value refinement. By the proof of Theorem~\ref{thm:nash}, $G^i(\omega) \geq 0$ with equality if and only if $w^1_1(\omega) = w^2_1(\omega)$. Since gains depend on $\omega$ only through the weight vectors, we write $G^i(\omega) = g^i(w^1_1, w^2_1)$ for some function $g^i: [0,1]^2 \to \mathbb{R}_+$.

We establish that $g^i$ is strictly increasing in preference divergence $|\Delta| = |w^1_1 - w^2_1|$. When $\Delta = 0$, both agents have identical preferences over dimensions, their indifference curves have identical slopes, and the bundled allocation $(1/2, 1/2)$ is Pareto efficient within $\mathbb{A}_1(\omega)$. The Nash solution coincides with bundling, yielding $g^i(w, w) = 0$.

When $\Delta \neq 0$, the agents' indifference curves have different slopes at the bundled allocation, creating a lens-shaped region of Pareto improvements. As $|\Delta|$ increases, this lens expands: the slopes diverge further, permitting more aggressive dimensional specialization while still benefiting both parties. The Nash solution selects a point in this lens that Pareto-dominates bundling. By the envelope theorem \citep{MilgromSegal2002}, the Nash payoff is continuous in the parameters $(w^1_1, w^2_1)$ determining the feasible set. Since larger $|\Delta|$ expands the set of Pareto improvements over bundling, $g^i$ is continuous and strictly increasing in $|\Delta|$ for $|\Delta| > 0$.

The variance of preference divergence satisfies
\[
\Var(\Delta) = \Var(w^1_1) + \Var(w^2_1) - 2\Cov(w^1_1, w^2_1) = 2(\sigma^2 - \gamma).
\]
With marginal distributions fixed, decreasing $\gamma$ increases $\Var(\Delta)$, spreading probability mass toward larger values of $|\Delta|$. Since $g^i$ is increasing in $|\Delta|$, the expected gain $\E[G^i] = \E[g^i(w^1_1, w^2_1)]$ is strictly increasing as $\gamma$ decreases.

For the maximum, the minimum feasible covariance given the marginals is $\gamma = -\sigma^2$, achieved when $w^2_1 = 1 - w^1_1$ almost surely. This represents perfectly negatively correlated preferences: when Agent~1 values dimension~1 highly, Agent~2 values dimension~2 equally highly. At this extreme, $|\Delta| = |2w^1_1 - 1|$ achieves its maximum in distribution, and each dimension is allocated entirely to the agent who values it relatively more, fully exploiting gains from specialization. \qedhere
\end{proof}


\bibliographystyle{chicago}
\bibliography{references}

\end{document}